\documentclass[10pt,journal,compsoc]{IEEEtran}

\ifCLASSOPTIONcompsoc
\usepackage[nocompress]{cite}
\else
\usepackage{cite}
\fi
\ifCLASSINFOpdf
\else
\fi
\usepackage{amsmath,amsfonts}
\usepackage{amsthm}
\newtheorem{proposition}{Proposition}
\newtheorem{lemma}{Lemma}
\usepackage{algorithmic}
\usepackage{algorithm}
\usepackage{array}
\usepackage[caption=false,font=normalsize,labelfont=sf,textfont=sf]{subfig}
\usepackage{textcomp}
\usepackage{stfloats}
\usepackage{url}
\usepackage{verbatim}
\usepackage{graphicx}
\usepackage{cite}
\usepackage{times}
\usepackage{epsfig}
\usepackage{amssymb}
\usepackage{color}
\usepackage{multirow}
\usepackage{epstopdf}
\usepackage{caption}
\usepackage{booktabs}
\usepackage{enumerate}
\usepackage{float}
\usepackage{ragged2e}
\begin{document}

\title{Language Knowledge-Assisted Representation Learning for Skeleton-Based Action Recognition}

\author{Haojun~Xu,
	Yan~Gao,
	Zheng~Hui,
	Jie~Li,
	and~Xinbo~Gao,~\IEEEmembership{Senior~Member,~IEEE}
\IEEEcompsocitemizethanks{\IEEEcompsocthanksitem Haojun Xu, Yan Gao, and Jie Li are with Xidian University. \protect\\
E-mail: \{haojunxu, yangao\}@stu.xidian.edu.cn; \protect\\ leejie@mail.xidian.edu.cn
\IEEEcompsocthanksitem Zheng Hui is with the Alibaba Group, Hangzhou 310000, China. \protect\\
E-mail: hz1406899875@gmail.com
\IEEEcompsocthanksitem Xinbo Gao is with Xidian University and with Chongqing University of Posts and Telecommunications. \protect\\
E-mail: xbgao@mail.xidian.edu.cn; gaoxb@cqupt.edu.cn \protect\\ (Corresponding author: Xinbo Gao.)}
}


\markboth{Journal of \LaTeX\ Class Files,~Vol.~14, No.~8, August~2015}%
{Shell \MakeLowercase{\textit{et al.}}: Bare Advanced Demo of IEEEtran.cls for IEEE Computer Society Journals}


\IEEEtitleabstractindextext{%
\begin{abstract}
\justifying
How humans understand and recognize the actions of others is a complex neuroscientific problem that involves a combination of cognitive mechanisms and neural networks. Research has shown that humans have brain areas that recognize actions that process top-down attentional information, such as the temporoparietal association area. Also, humans have brain regions dedicated to understanding the minds of others and analyzing their intentions, such as the medial prefrontal cortex of the temporal lobe. Skeleton-based action recognition creates mappings for the complex connections between the human skeleton movement patterns and behaviors. Although existing studies encoded meaningful node relationships and synthesized action representations for classification with good results, few of them considered incorporating a priori knowledge to aid potential representation learning for better performance. LA-GCN proposes a graph convolution network using large-scale language models (LLM) knowledge assistance. First, the LLM knowledge is mapped into a priori global relationship (GPR) topology and a priori category relationship (CPR) topology between nodes. The GPR guides the generation of new ``bone'' representations, aiming to emphasize essential node information from the data level. The CPR mapping simulates category prior knowledge in human brain regions, encoded by the PC-AC module and used to add additional supervision—forcing the model to learn class-distinguishable features. In addition, to improve information transfer efficiency in topology modeling, we propose multi-hop attention graph convolution. It aggregates each node's k-order neighbor simultaneously to speed up model convergence. LA-GCN reaches state-of-the-art on NTU RGB+D, NTU RGB+D 120, and NW-UCLA datasets.
\end{abstract}

}

\maketitle
\IEEEdisplaynontitleabstractindextext
\IEEEpeerreviewmaketitle


%
%
%
%

\section{Introduction}
\begin{figure}[h]
	\centering
	\includegraphics[width=\linewidth]{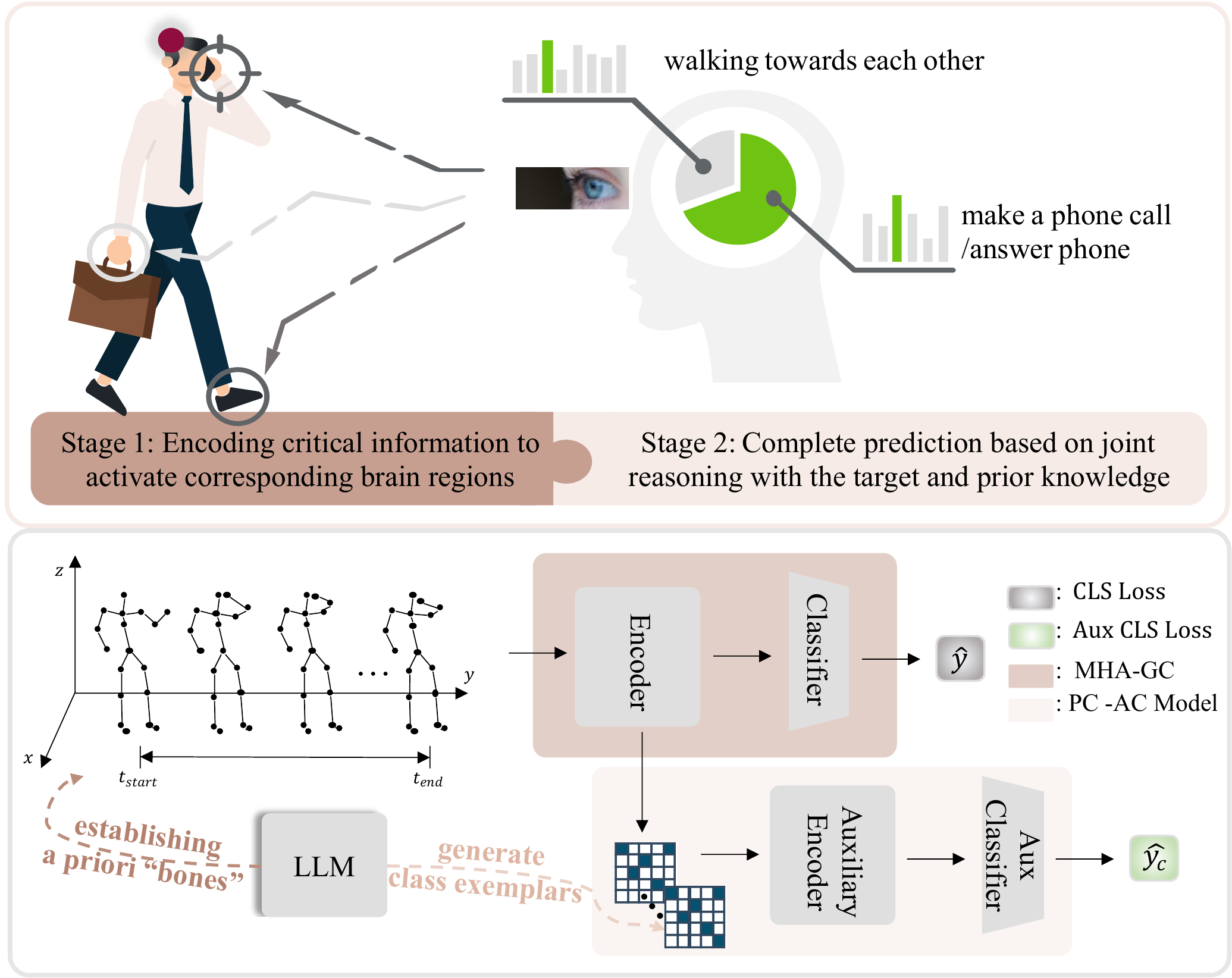}
	\caption{Schematic of LA-GCN concept. The top half of this figure shows two brain activity processes when humans perform action recognition. The bottom half shows the proposed multi-task learning process. The knowledge of the language model is divided into global information and category information to simulate the a priori knowledge used in human reasoning to aid the model. The encoder infers the correlation between joints and thus refines the topology using contextual information.}
	\label{fig:pipline}
\end{figure}
\IEEEPARstart{H}{uman} action recognition, one of the core tasks of video understanding, is a classification task that observes and infers agent behavior from a third-person perspective \cite{pami_SunKRBWL23_Review}. This task can be widely used in human-computer interaction \cite{LI_2019_Interaction}, video surveillance \cite{2021_Emergency}, healthcare \cite{pami_YuLZZC23_MMNet}, and short entertainment videos \cite{acm_18_recommender_system}. Among them, skeleton-based action recognition \cite{Chi_2022_InfoGCN,pami_SongZSW23_EfficientGCN,Chen_2021_CTRGCN,cvpr20_PengfeiZhang_SemanticsGuided,aaai_chen2021_multi,Liu_2020_MSG3D,Ye_2020_dynamicGCN,Shi_2019_2sAGCN,Thoker_2021_Contrastive3D,Su_2021_Uncertainty3D} is widespread because of its robustness to various environmental noises in the video and its high compactness that facilitates model focus. Graph convolutional network (GCN) based methods \cite{Yan_2018_STGCN,Shi_2019_2sAGCN} made revolutionary advances in this field.

The human musculoskeletal system allows body parts to move and thus perform different actions \cite{agur2009grant,nordin2001basic}. The skeleton-based data modality conforms to the human anatomy \cite{Yan_2018_STGCN}, thus making the learning of GCN more interpretable. It contains only 2D or 3D coordinates of the main joints of the human body, allowing the model to recognize human movements by reasoning about the skeleton sequences. In addition, the skeleton modal is more privacy-friendly than other modalities.

This paper introduces a language model knowledge-assisted graph convolutional network (LA-GCN) to enhance skeleton-based action recognition. Inspired by current cognitive neuroscience research \cite{finisguerra2019non,wurm2018role,NeuroSci_2018_Teory_of_Mind} and benefiting from the development of Large-scale Language Model (LLM) \cite{naacl_DevlinCLT19_BERT,nips_Brown20_GPT3,icml_Radford21_CLIP}, LA-GCN uses a large-scale textual knowledge base to simulate the brain regions that the human brain uses to accomplish behavioral prediction to help GCN network make predictions.

As shown in the upper part of Fig. \ref{fig:pipline}, when observing the behavior of others, the temporoparietal joint area in the brain is stimulated to activate the corresponding brain area associated with the current action, and the prediction of that action is accomplished by reasoning based on a priori knowledge and goals \cite{NeuroSci_2018_Teory_of_Mind}. Much research \cite{pami_Zhang19_VALSTM,pami_SongZSW23_EfficientGCN,Yan_2018_STGCN,Ye_2020_dynamicGCN,Chen_2021_CTRGCN,Chi_2022_InfoGCN,Shu_pami2023_MGAC,pami_WenGFZXL23_MotifGCNs} has tried to model essential joint information to enhance topology learning, but it still needs action-related a priori information to assist. Therefore, the proposed LA-GCN has two parts, the global prior relationship topology (GPR Graph) between nodes and the category prior relationship topology (CPR Graph) between nodes obtained from the LLM to assist the GCN in action reasoning. First, the GPR Graph is used to guide the generation of a new skeleton modal to assume the function of modeling critical information from the data level and using this modal as the input to the ``neuronal cluster'' GCN for feature aggregation. Then, the CPR Graph is used to compose an a priori consistency-assisted classification (PC-AC) module to help model learning based on features with enhanced semantic relationships. The framework of our proposed approach is shown at the bottom of Fig. \ref{fig:pipline} for the ``make a phone call'' action.

Our GPR-Graph in LA-GCN contains relatively well-established meaningful inter-node relationships, even if two nodes are spatially distant. Specifically, we use the text encoder of LLM to extract node features for each joint and establish correlations by finding class centers for all joints. We borrow the BERT \cite{naacl_DevlinCLT19_BERT} widely used for feature extraction in language text as our LLM model. The GPR-Graph generates new skeleton data for the GCN network by preserving the key ``bones'' according to the rules. The global prior connection of the new skeleton representation can reduce the difficulty of topology modeling and get the differentiated feature representation.

CPR Graph is a mapping of LLM's prior knowledge for action classes. Meanwhile, our PC-AC module aims to simulate how humans think and reason with a priori knowledge. Therefore, the PC-AC module encodes the CPR Graph as a category template topology to add additional supervision for GCN. It is suitable for solving some challenging process-like action classification problems, such as ``reading'' and ``writing'' which are similar in node relationships.

In addition, we propose a new feature aggregation method, a multi-hop attention graph convolution (MHA-GC) block, for improving the efficiency of message passing between nodes. When GC performs normal messaging, feature aggregation of a node is contributed by the directly connected nodes, which are called one-hop neighbors. However, as the depth of the GCN increases, the limitations of one-hop messaging increase layer by layer, leading to message delays and semantic over-smoothing that are detrimental to constructing inter-node contexts \cite{AAAI18_Li_DeeperInsights,aaai_Liu19_GeniePath,icml_Xhonneux20_Continuous}. For this reason, we use multi-hop attention in a single GC layer to establish remote node relationships. Specifically, we spread the computation of attention from each node to the nodes indirectly connected to it, and the attention is represented using the distance between node features. As a result, MHA-GC can strengthen the semantic relationships between nodes with connectivity and accelerate model convergence.

We evaluated three skeleton-based action recognition datasets NTU RGB+D 60 \cite{Shahroudy_2016_NTURGBD}, NTU RGB+D 120 \cite{Liu_2020_NTURGBD120}, and NW-UCLA\cite{Wang_2014_NWUCLA}. The performance on cross-subjects split for the first two benchmarks is 93.5\% and 90.7\%. On NW-UCLA, we have 97.6\% of top1 accuracy. The experiments show that the proposed LA-GCN outperforms the state-of-the-art techniques. Our main contributions are summarized as follows:
\begin{itemize}
	\item An LA-GCN is proposed to use the prior knowledge of LLM to assist the GCN for skeleton-based action recognition.
	\item A new skeleton representations method is proposed for GCN models ensemble. GPR Graph with global information is performed in this method to reduce the topological modeling difficulty.
	\item An auxiliary supervised module PC-AC with class information encoding is proposed to improve the recognition rate of similar actions.
	\item A new multi-hop attention feature aggregation method, MHA-GC, is proposed to improve the model's information transfer efficiency and accelerate model convergence.
\end{itemize}

\section{Related Work}
\subsection{Topology Construction}
Early skeleton-based action recognition methods contain CNN-based \cite{Kim_2017_TCN,Li_2018_CNN,Ke_2018_RotClip} and RNN-based \cite{Liu_2016_ST_LSTM,Lee_2017_TSLSTM,Si_2019_RNN,pami_Zhang19_VALSTM} methods. However, it is still necessary to explore the skeletal data structure adequately.
Topology modeling is the design focus of GCN-based methods \cite{pami_Zhang19_VALSTM,pami_SongZSW23_EfficientGCN,Yan_2018_STGCN,Ye_2020_dynamicGCN,Chen_2021_CTRGCN,Chi_2022_InfoGCN,pami_WenGFZXL23_MotifGCNs}. 
The pioneer of graph convolution, ST-GCN \cite{Yan_2018_STGCN}, predefines the topology based on the human structure as the input to GCNs. The topology is fixed in both the training and testing phases. Based on this, multi-scale graph building is introduced to GCNs for multi-level joint relationship modeling \cite{Liu_2020_MSG3D}. There are limitations in the inference performance of static methods due to their inability to follow the network co-optimization. Some works \cite{Shi_2019_2sAGCN,cvpr20_PengfeiZhang_SemanticsGuided,Chi_2022_InfoGCN} augment topology learning using self-attention mechanisms to model the correlation between two joints given the corresponding features. Topology learned using local embeddings still needs to satisfy the need of GCN for high-quality node relationship modeling. Dynamic GCN \cite{Ye_2020_dynamicGCN} uses contextual features of all joints learned together to obtain global correlations.
To make the features used for modeling more differentiated, CTR-GCN \cite{Chen_2021_CTRGCN} designs channel-specific topology maps to explore more possibilities for feature learning in different channels.
Shift-GCN \cite{Cheng_2020_ShiftGCN} removes the limitations of predefined topological graphs and uses shifting operators for inter-joint feature fusion, learning node relationships implicitly.
Since the action changes in real time, the dynamic approach is more advantageous and generalizes better than the static approach. In this paper, our LA-GCN topology modeling belongs to the dynamic mode.

\subsection{Language Model in Skeleton-Based Action Recognition}
The development of natural language processing (NLP) tasks gave birth to the pre-trained representation model BERT \cite{naacl_DevlinCLT19_BERT} from transformers with a bi-directional encoder. It is a solution to solve NLP tasks using pre-trained LLMs to finetune. However, this solution could be more efficient and can only be adapted to one task at a time. In order to use pre-trained LLMs more efficiently, prompt learning (PL) emerged as a technique capable of adapting different tasks to large models. Specifically, the PL technique adapts the model to a new downstream task by adding specific textual parameters to the input of LLMs based on the definition of new tasks, significantly increasing the efficiency of knowledge utilization in LLMs. Meanwhile, related approaches such as CLIP \cite{icml_Radford21_CLIP} and \cite{pmlr-v139-jia21b} have successfully applied PL to learning computer vision (CV) downstream tasks and demonstrated powerful graphical and textual representation capabilities. This brings light to the CV field to step into a new phase of visual text.

For the action recognition task, ActionCLIP \cite{wang2021actionclip} uses the CLIP training scheme for video action recognition and adds a transformer layer to guarantee the temporal modeling of video data. In the construction of PL templates, ActionCLIP directly uses class labels as input text to construct cues such as ``[action] of video,'' ``carry out [action] of person,'' and so on. 
LST \cite{xiang2022language} uses LLM for skeleton-text multimodal representation learning in the skeleton-based action recognition task. The PL technique in LST is mainly used as a skeleton text pair building, i.e., using the prompt to allow detailed descriptions generated by GPT3 \cite{nips_Brown20_GPT3} for each class of skeleton actions. 
Inspired by human cognitive reasoning processes in cognitive neuroscience, our approach uses LLM knowledge to model human knowledge of brain regions in action reasoning. PL is used to construct topological maps for assisted learning, which contain fine-grained semantic relationships between human actions and joints.

\section{Preliminaries}
The GCN family of methods \cite{Yan_2018_STGCN,Shi_2019_2sAGCN,Liu_2020_MSG3D,Ye_2020_dynamicGCN,Chen_2021_CTRGCN,Chi_2022_InfoGCN} models the human skeleton in space and time, respectively. Precisely, these methods follow the work of \cite{Yan_2018_STGCN} by first converting the skeleton data into a spatiotemporal graph $ G = (V, E) $, where $ V $ denotes joints, and $ E $ is edges, i.e., the physical connection between joints, represented by the graph's adjacency matrix $ A $. Spatial modeling aggregates node features based on the adjacency matrix and use them to represent the relationships between joints in the skeleton graph. Temporal modeling is then used to model the motion patterns of the joints.

If the skeleton-based action recognition task is symbolized can be expressed as follows:  $ \mathbb{R}^{N \times T \times  d} \to  \mathbb{R}^{p} $, where it has $ N $ joints and $ T $ frames, and the dimension of each joint coordinate is $ d $. The GCN finally outputs the probability value that the action belongs to each class, and there are $ p $ classes in total. For each frame, the graph convolution can be expressed as:
\begin{equation}
	\label{formula:gcn}
	F^{l+1} = \sum_{s \in S} \tilde{A_s} F^{l} W_s,
\end{equation}
where $F^{l} \in \mathbb{R}^{N \times C_1}$ is denoted as the input feature of channel $ C_1 $ and $F^{l+1} \in \mathbb{R}^{N \times C_2}$ is denoted as the output feature of channel $ C_2 $. $ S $ defines the neighborhood of each joint node, and $S=\{\text{root}, \text{centripetal}, \text{centrifugal}\}$, where the root is itself, the centripetal group is nodes near the skeleton center, and the centrifugal group is nodes far from the center. 
$\tilde{A_s} = \Lambda_s^{- \frac 1 2} A_s \Lambda_s^{- \frac 1 2}$ is normalized matrix $ A \in \{0, 1\}^{N \times N} $, and the diagonal matrix $\Lambda_s$ in $\tilde{A_s}$ is define as $\Lambda_s^{ii} = \sum_j(A_s^{ij}) + \alpha$ to prevent empty row, $ \alpha $ is a small positive number, say $0.001$. $W_s \in \mathbb{R}^{1 \times 1 \times C_1 \times C_2}$ is the weight of every $ s $. The graph convolution of most GCN-based methods on the $ T $ dimension acts as the rule 1D convolution.

\section{LLM Guided Topology Assistant Graph Convolution Network}
\label{Sec:method}
LA-GCN is a novel learning framework that utilizes LLM model knowledge for reasoning about the actions of a given skeleton sequence. We first integrate a priori knowledge into a multimodal skeleton representation (Sec. \ref{sec:GPR} and \ref{sec:multimodalrp}). For better feature aggregation of the prior representation, we also introduce a neural architecture (Sec. \ref{sec:encoding}) and a loss of category-based prior knowledge for multi-task learning (Sec. \ref{sec:pc-ac}). Finally, the overall learning scheme is given. \footnote{Note that all symbols used in this section are summarized in Table \ref{tab:notations} of the Appendix.}

\subsection{A Global Prior Relation Graph Generated by LLM}
\label{sec:GPR}
We will first extract text features for each class of action labels and all the joints using a large-scale pre-trained network Bert. The loss of BERT \cite{naacl_DevlinCLT19_BERT} during pre-training consists of completing and entering two sentences to predict whether the latter one is the following sentence of the previous one.
The output of the last layer of the model is used for the completion of the blank, and the output of the Pooler part is used for the next sentence prediction.
\begin{figure}[h]
	\centering
	\includegraphics[width=\linewidth]{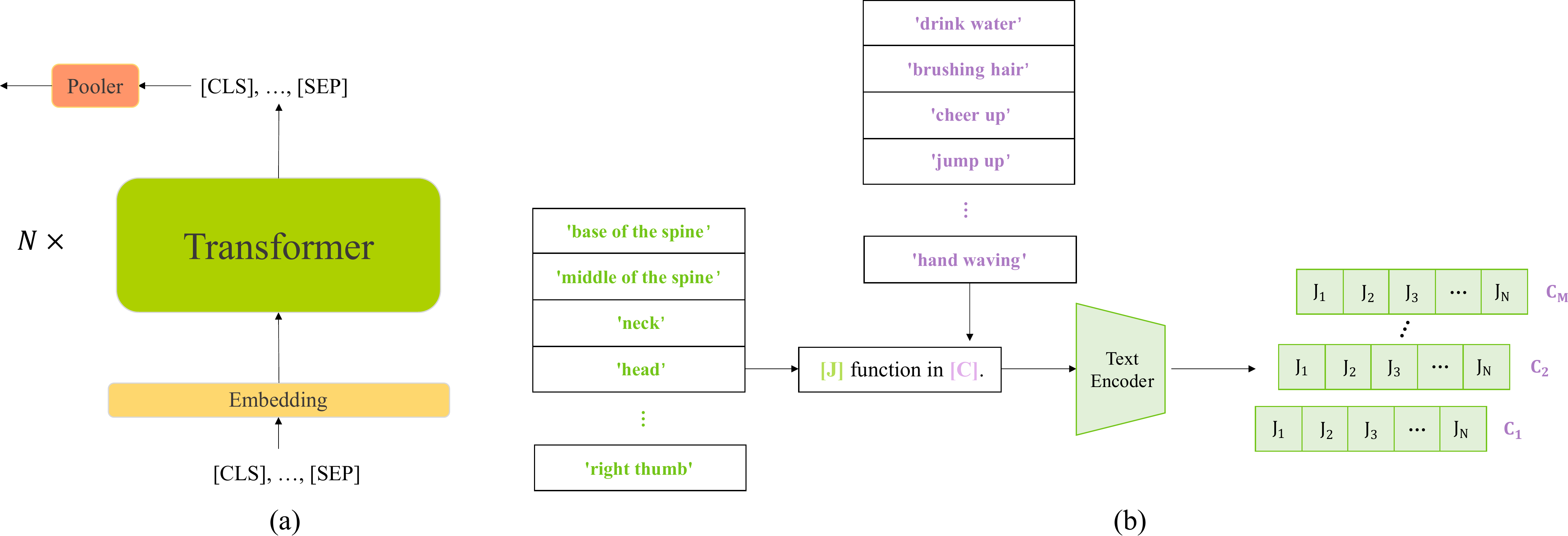}
	\caption{Extraction of text features. Subfigure (a) is Bert's architecture. (b) Our method uses the learned text encoder to extract text features by embedding the names of classes [C] and the names of all joints [J] of the target dataset.}
	\label{fig:txt_feature_extract}
\end{figure}

As shown in Fig. \ref{fig:txt_feature_extract}(a), Bert first tokenizes the input sentences and then adds [CLS] and [SEP] tokens at the end of the sentences. The tokens in the sentence are the indexes of the corresponding words in the vocab. After getting these indexes, we transform them into continuous vectors by the embedding layer. The embedding layer can be regarded as a lexicon-size vector library, and the embedding process indexes the vectors according to their indexes. The output of the last layer of the model can be obtained after N transformer layers: the features of each token and the features of [CLS] after Pooler. Since the features after Pooler contain the semantics of the whole sentence, they can directly use when doing tasks such as text classification in general. Therefore, in this paper, the features after Pooler are also selected as the features of our skeleton nodes. Given $ M $ action categories and $ N $ human nodes, our specific process of extracting features for each node is shown in Fig. \ref{fig:txt_feature_extract}(b). We feed the contained class name and joint name text, e.g., [joint] function in [class], into the text encoder of LLM to get the corresponding output features C $ \in \mathbb{R}^{N \times C} $ for all joint tokens of each action category, where $ C = 256 $ is the feature dimension of node feature J.
\begin{figure}[h]
	\centering
	\includegraphics[width=0.8\linewidth]{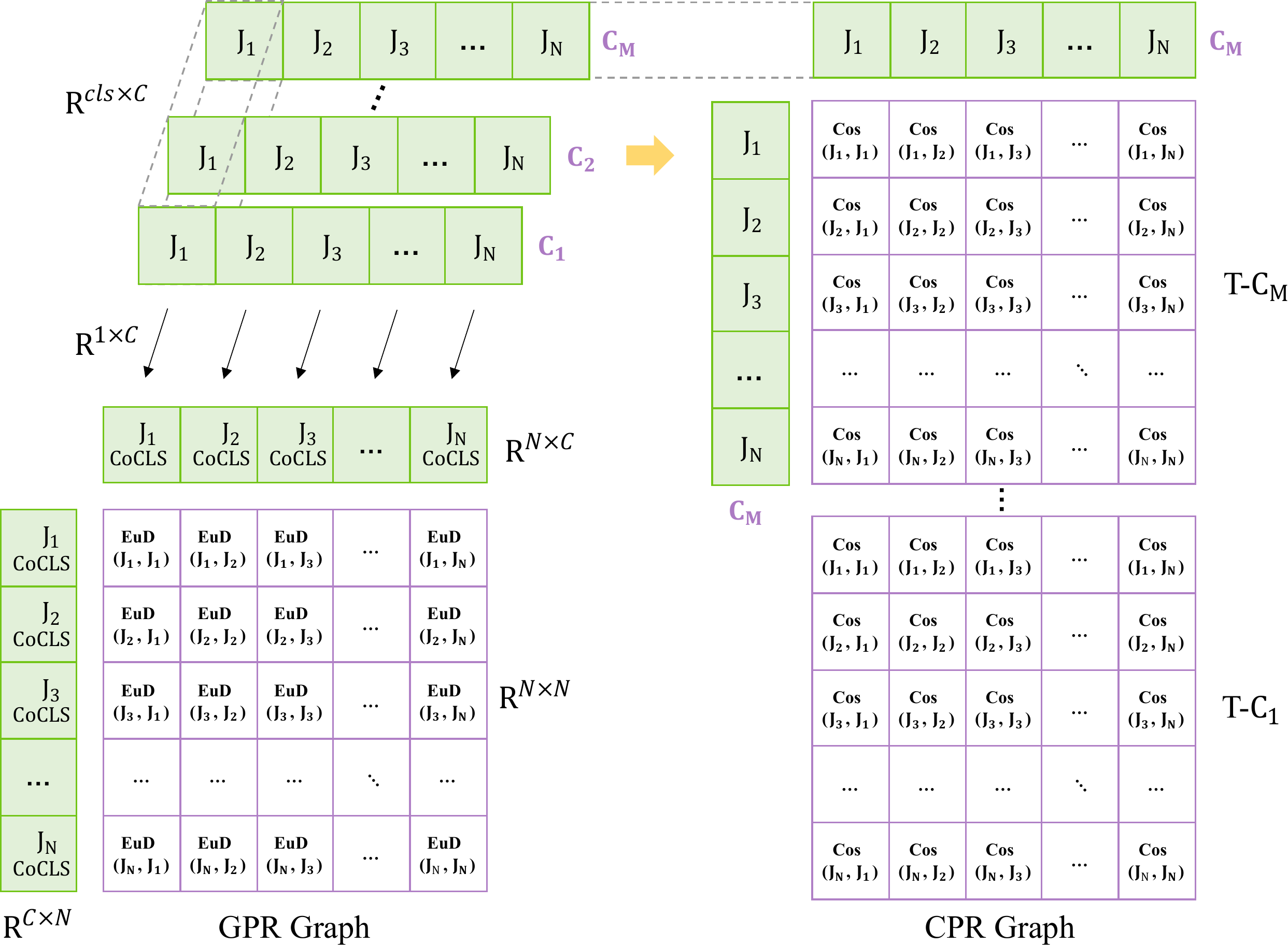}
	\caption{Summarize our approach to generate prior topologies. GPR Graph is obtained by computing the class centers of the joints and computing correlations between the node feature of each action to obtain the CPR Graph.}
	\label{fig:gen_graph}
\end{figure}

Then, we want to obtain a global inter-node relationship topology graph GPR-Graph containing the semantic knowledge in LLM to guide the skeleton representation generation. As shown in the left part of Fig. \ref{fig:gen_graph}, to obtain the GPR Graph, we first have to find the centroids (CoCLS) of the text features J $ \in \mathbb{R}^{1 \times C} $ of each node on the class dimension. Specifically, the features $ \text{J}_i \in \mathbb{R}^{cls \times C}, i \in (1, N) $, and $ cls = M $ of the same node on different categories are averaged to obtain the category centroid vector $ \text{J}^{ CoCLS }\in \mathbb{R}^{N \times C}$, and $\text{J}_{i}^{CoCLS} = \frac{1}{M} \sum_{j=1}^{M} \text{J}_{i}^{C_j} $. Then correlations are calculated between nodes, and the similarity measure chooses the Euclidean distance. The final GPR Graph $ \in \mathbb{R}^{N \times N} $ with global prior information is obtained, corresponding to the distance between the node class centers.

\subsection{A Priori Skeleton Modal Representation}
\label{sec:multimodalrp}
In this section, we introduce a form of skeleton representation generated using a GPR Graph called a priori skeleton modal representation. Consistent with previous work \cite{Shi_2019_2sAGCN,Liu_2020_MSG3D,Chen_2021_CTRGCN,Chi_2022_InfoGCN}, we train our model to complete inference using multiple modal representations. The prior representation primarily utilizes the relative positions of joints for complementary learning, i.e., representing the data as joints and bones. In this case, the bone feature is a transformation of the joint feature: it is obtained by subtracting the starting node from the ending node according to the physical connection \cite{Yan_2018_STGCN}. In detail, the joint-bone relationship at the moment $ t $ can be expressed as:
\begin{equation}
	\label{formula:joint_bone}
	\widetilde{X_t} = (I - B)X_t,
\end{equation}
where $ B \in \mathbb{R}^{N \times N} $ denotes the bone matrix containing the position relationship between the source and target joints, which is a binary matrix with $ B_{ij} = 1 $ if the $ i $-th joint is the source of the $ j $-th joint and otherwise $ B_{ij} = 0 $, where the row corresponding to the base point is a zero vector. This processing significantly improves the performance of action recognition, which means that bone representations with considerable differences from the joints can learn more differentiated features to complement. In addition, because of the high weight of physical bone topology in relationship modeling, reasonable skeleton construction is critical to learning inter-node relationships.

We consider the difference between nodes as ``bone'' data for additional representation. For instance, the NTU datasets \cite{Shahroudy_2016_NTURGBD,Liu_2020_NTURGBD120} include 25 nodes, and if we regard the difference vector between any two nodes as a ``bone'', there are $ C_{25}^2=300 $ ``bones.'' According to the nature of the bone matrix $ B $, if we want to construct a matrix that is different from the bone but has the same characteristics as the bone matrix, we need to transform the problem into an ``alignment problem'': Pick a ``bone'' connected to each joint except the base point that cannot be repeated. In other words, we will bind the degree of entry for each joint.

We extracted the ``bone'' data of all samples in the NTU+D 60\cite{Shahroudy_2016_NTURGBD} training set.\footnote{the average standard deviation of all 300 ``bones'' extracted from the training set is 0.1073.} 
After calculation, the bone links in the physical bone matrix \cite{Yan_2018_STGCN} have a small sum of standard deviations. Based on this, we define the ``bone'' selection function as $ g(\cdot) $, the new ``bones'' are represented as $\widetilde{B} = g_{min}(B_{std})$, where $ B_{std} $ is the sum of the standard deviations of the selected bones and $ g_{min} $ means that the links in the smallest $ B_{std} $ is chosen as the new ``bone'' representation. Specifically, given a bone vector $ \tilde{x_t} $ represented in Eq. \ref{formula:joint_bone}, its skeleton standard deviation (std) is the average of the std of three coordinates in the time dimension: $ b_{std} = mean(\sigma_x(\tilde{x_t}), \sigma_y(\tilde{x_t}), \sigma_z(\tilde{x_t})) $. We visualize the selected $ \widetilde{B} $ in Fig. \ref{fig:build_new_bone}, while other designs are given for comparison.
\begin{figure}[h]
	\centering
	\includegraphics[width=0.9\linewidth]{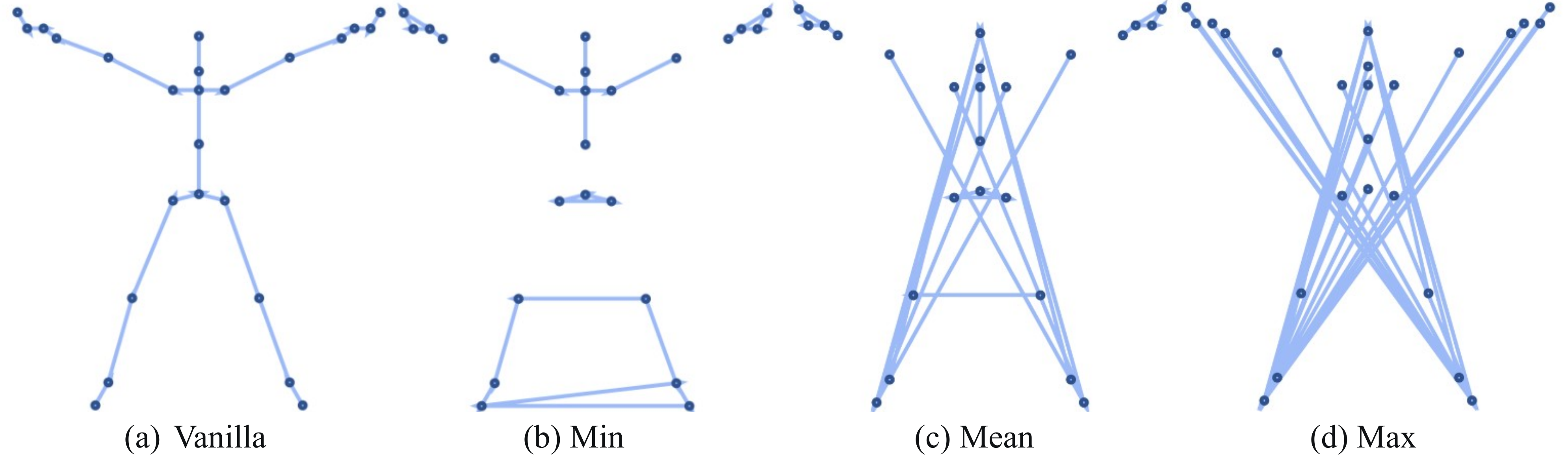}
	\caption{Multimodal representation of the skeleton. The arrows depict the ``bone'' vector from the source to the target joint. (a) is the bone matrix \cite{Yan_2018_STGCN}. (b), (c) and (d) is our minimum, mean, and maximum std summation matrix, respectively. It was found that nearly half of the skeletons in (b) are consistent with (a) but contain more detailed relationships, and (c) implicitly contains information about the connections between body parts, such as hands, torso and hands, torso and legs, and legs. Where (b) is used as our new ``bone'' representation.}
	\label{fig:build_new_bone}
\end{figure}

Nearly half of the bones in $ \widetilde{B} $ in Fig. \ref{fig:build_new_bone}(b) are consistent with the original bone matrix but focus more on some complicated relationships, such as fingertips, toe tips, and hip joints, and have symmetry.\footnote{The summation of standard deviation from (a) to (d) in Fig. \ref{fig:build_new_bone} are: 0.8265, 0.6977, 2.2312, and 4.1997.} Further, we use the inter-node distance of the GPR Graph in Sec. \ref{sec:GPR} to weight the bone vector and extract the bones with the minimum std summation as the new skeleton representation. Our a priori modal representation has two main advantages: (1) It captures interactions between remote nodes that are not directly connected. The interaction may help recognize some actions. (2) It contains context-dependence on class properties.

\subsection{Feature Aggressive}
\label{sec:encoding}
We introduce a multi-hop attention mechanism to model the context dependence of joints, adding further information about neighboring nodes in the topology learning process. This is considering enhancing the learning of long-distance relationships present in the a priori skeletal modal representation. Fig. \ref{fig:la-gc} provides an overview of the encoder-classifier structure of LA-GCN.
\begin{figure}[h]
	\centering
	\includegraphics[width=0.85\linewidth]{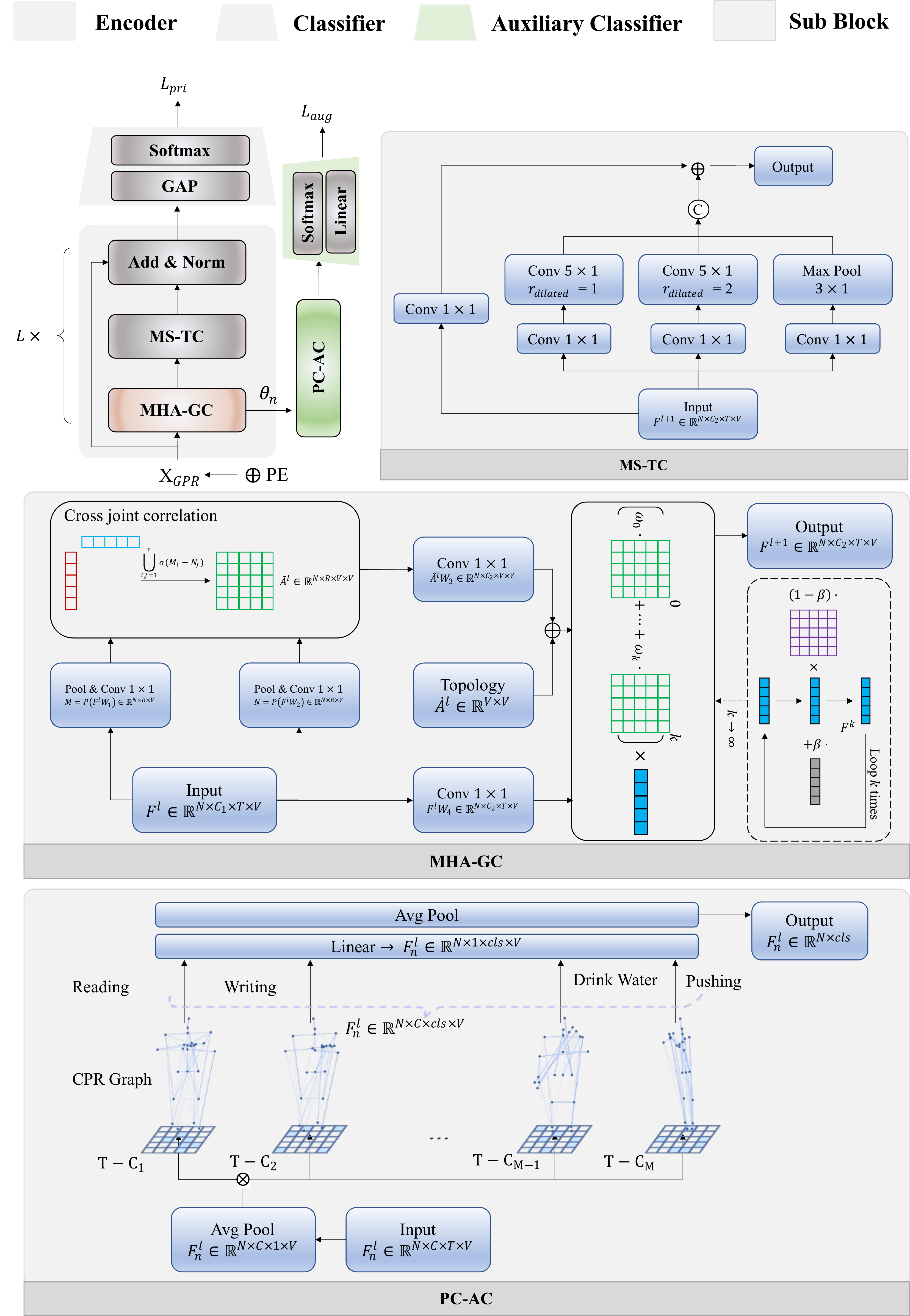}
	\caption{LA-GCN architecture. The model consists of a main branch and an auxiliary branch containing an encoder and a classifier. The details of each sub-block structure are shown in the gray box. The symbols $ N $, $ C $, $ T $, and $ V $  in this sub-block indicate the batch size, the number of input/output channels, the number of video frames, and the number of joints, respectively. We use the GPR Graph to generate a skeleton representation $ \text{X}_{GPR} $ as input for learning the topology. This allows the encoder to capture the inter-node context directly with global prior semantic information. The PC-AC guides our neural structure to learn conditional representations based on the class prior node relations.}
	\label{fig:la-gc}
\end{figure}

\subsubsection{More Information for Learnable Topology} 
When relying solely on the physical topology of the skeleton to aggregate features, information learning suffers from latency and dilution effects \cite{Yan_2018_STGCN}. As the relationship between joints changes with the execution of actions, relationship learning between unnaturally connected nodes is affected by node spacing. These two issues also affect each other in the aggregation process of raw GC. For illustration, two remote nodes have just established a connection at some point, and they may be diluted or lost after averaging and non-linear transformation operations in the neighborhood (see Eq.\ref{formula:gcn}). The information delay will be further aggravated.

Some approaches \cite{Chen_2021_CTRGCN,Chi_2022_InfoGCN} use an attention mechanism to guide the learning of the internal topology to complement the node relations adaptively. However, this is not optimal because it contains only first-order neighbor information in the semantic space. The neighborhood information indicated by the semantic space needs to be sufficiently learned, which is especially true for our a priori modal representation with contextual information. Like the video understanding task, allow the network to focus first on extracting intra-frame features and then on the fusion of inter-frame features \cite{mm_SnoekWS05_EarlyVersusLateFusion}. Here we want to allow the nodes to learn more fully before integrating. Therefore, we propose an architecture that uses a multi-hop attention mechanism to capture each node's neighborhood to improve the efficiency of information exchange and speed up the convergence of the GCN.

\subsubsection{Pre-preparation of The Input Data}
Before the skeleton sequence $\text{X}$ is fed into the encoder for feature aggregation, we need to make two preparations: one is to weight $\text{X}$ using the GPR Graph as mentioned in Sec. \ref{sec:multimodalrp}, and the other is to add the position embedding (PE) after $\text{X}$ has been linearly transformed to obtain the embedding representation of the nodes. The $ PE $ is the embedding that contains the abstract position of each node in all nodes $ V $. This gives us the initial feature representation:
\begin{equation}
	\label{eq:layer0_input}
	F^{(0)} = \text{X}_{GPR}W_0 + PE,
\end{equation}
where $ W_0 $ is the parameter matrix, $ F^{(0)} \in \mathbb{R}^{N \times C \times T \times V} $, and $ PE \in \mathbb{R}^{C \times V} $. If using the raw skeleton modal, the $ \text{X}_{GPR} $ in Eq. \ref{eq:layer0_input} is changed to X. 

\subsubsection{Encoder Structure}
The base block of the encoder consists of two sub-blocks MHA-GC and MS-TC responsible for the spatial and temporal modeling, respectively. As shown in Fig. \ref{fig:la-gc}, the output features $ F^{l} $ of the base block are obtained by adding the hidden layer features output by the two submodules using BN normalization with the upper layer feature $ F^{l-1} $ after residual connection.

\noindent\textbf{MS-TC} We use the commonly used multi-scale temporal encoding module \cite{Chi_2022_InfoGCN,Chen_2021_CTRGCN,Liu_2020_MSG3D} to model skeletal sequences with different time lengths. The module contains two dilated convolutions with different kernel settings and a max pooling layer. The output is obtained by adding the skip connections with the $ 1 \times 1 $ convolution.

\noindent\textbf{MHA-GC} We propose a new multi-hop attention graph convolution module for inter-node relationship modeling called MHA-GC. Computing multi-hop attention \cite{ijcai_WangY0L21_MultiHop} for complementary neighborhood node information, MHA-GC first computes the first-order attention on all nodes $ \bar{A}^{l} $ for multi-hop diffusion.

As shown in Fig. \ref{fig:la-gc}, at each layer $ l $, the feature $ F^{l} $ is passed through two branches containing the transformation equation P, consisting of a pooling layer and a $ 1\times 1 $ convolution, respectively. Feature vectors  $ M, N \in \mathbb{R}^{N \times R \times V} $ are obtained after P, where $ R $ is the reduced-to-feature dimension. Any pair of vertices $ (v_i, v_j) $ in $ M $ and $ N $ is computed separately to obtain the first-order neighborhood information $ \tilde{A}^{l} = \bigcup^{v}_{i,j = 1} \sigma(M_i-N_j) $, where $ \sigma $ is the activation function, $ v $ indicates all nodes, and $ \tilde{A}^{l}_{ij} $ denotes the messages aggregation from node j to node i.

We combine the shared topology $ \dot{A}^{l} \in \mathbb{R}^{ V \times V } $ with the linear mapping of learnable attention $ \tilde{A}^{l} \in \mathbb{R}^{R \times V \times V} $ to obtain the refined attention:
\begin{equation}
	\bar{A}^{l}= \dot{A}^{l}+\gamma\tilde{A}^{l}W_{3},
\end{equation}
where $ \gamma $ and $ W_3 $ are the refinement weight parameters and mapping layer weights, respectively. The refinement method of $ \bar{A}^{l} $ and the generation method of the GPR Graph are based on calculating feature distance between nodes to ensure the consistency of feature semantics.

Then, the attention diffusion module calculates the attention between indirectly connected node pairs through a diffusion process based on the first-order attention matrix. Our multi-hop attention is computed as follows:
\begin{equation}
	\label{eq:relation_multi_one}
	\bar {\mathcal{A}} = \sum^{k}_{i=0} \omega_i \bar{A}^{i}, 
\end{equation}
where $ \omega_i = \beta(1- \beta)^{i} $, $ \beta \in (0, 1] $ is the decay factor of $ \bar{A} $, and $ \omega_i > \omega _{i+1} $; $ \bar{A}^{i} $ is the power of the matrix $ \bar{A} $. The original GC layer passes information in a one-hop attention weighting pattern (see Eq. \ref{formula:gcn}). The power matrix $ \bar{A}^{i} $ gives node information paths of length i, increasing the receptive field for associating relations between nodes. The implementation of $ \omega $ is based on inductive bias, which determines the weight of the contribution of node $ j $ acting on each node $ i $ on the path, and the larger the number of hops, the smaller the weight. The overlap of node aggregation on the same path also alleviates the information dilution effect of the original GC layer aggregating two remote nodes.

Finally, we define the feature aggregation function of the $ \bar{\mathcal{A}} $-based MHA-GC as
\begin{equation}
	F^{l+1} = \sigma(\bar{\mathcal{A}}^lF^lW_4^l),
\end{equation}
where $ \sigma $ denotes the activation function ReLU \cite{Nair_2010_ReLU} and $ W_4 $ is the weights of the output layer.
 
\subsection{A Priori Consistency-Assisted Classification Module}
\label{sec:pc-ac}
The goal of our a priori consistency-assisted classification (PC-AC) module is to design a class relationship topology graph containing a priori node relationships related to the predicted actions to help the GCN perform node feature aggregation. Specifically, the PC-AC module adds additional branch co-supervision to the GCN at training time, forcing each feature that passes through a branch containing a specific category topology to be predicted for that class.
 
As shown in Fig. \ref{fig:gen_graph}, we generate the category topology graph T-C using the text features constructed in Sec. \ref{sec:GPR}.
For each action, the text features of N nodes are combined two by two to calculate the similarity. We refer to T-C as a class topology exemplar and assume that they contain the node relationships that should be present to identify an action.
 
Next, to achieve our goal, LA-GCN will train the main and PC-AC auxiliary branches in a standard multi-task learning setup. We represent the entire multi-task model with a parameter $ \theta $ as a function $ f_\theta(x) $ and the input is $ \text{X} $. This parameter $ \theta $ will be updated by the common loss of the primary and auxiliary branches. The input features are obtained using a hard parameter-sharing approach \cite{Zhang_2018_Auxiliary}. Specifically, the category prediction of the auxiliary branch uses a set of features shared by the main branch $ \theta_n $, as shown in Fig. \ref{fig:la-gc}, where n is the selected main branch layer. After the final feature layer, we applied a task-specific layer to output the corresponding predictions, with the softmax classifier used in both branches. In this case, the classification head of the main branch of the model consists of an average global pool, and the classification head of the secondary branch consists of a fully connected layer and an average pooling layer.
 
We denote the main prediction as $ {f_\theta}^{pri}(x) $, the auxiliary prediction as $ {f_\theta}^{aux}(x) $, and the ground truth label as $ \hat y $. Given the corresponding layer features $ F_{n}^{l} $, and the category topology T-C, the cross-entropy loss in terms of the auxiliary branch is defined as:
\begin{equation}
	\label{formula:pc-ac loss}
	L_{aug} = -\sum_{k} y_k \cdot \mathit{log}(\hat{y_k}),
\end{equation}
where $ y_k $ is the one-hot vector containing the accurate category $ c $ of the sample $ \text{X} $, and $ \hat{y_k} = \frac{e^{z_i}}{\sum_{k} e^{z_k}}$, logit $ z $ is the output of feature $ Z_p $ containing the category prior after our auxiliary classification head, while $ Z_p \in \mathbb{R}^{N \times C \times cls \times V}$ is obtained by multiplying $ F_{n}^{l} \in \mathbb{R}^{N \times C \times 1 \times V} $ and T-C $ \in \mathbb{R}^{cls \times V \times V} $ in Fig. \ref{fig:la-gc}. The PC-AC module aims to maximize the likelihood that the output logit $ z_i $ belongs to category $ c $. For the primary supervision of the predicted action labels, logit is then the output of the main branch classification head. Ultimately, to update the parameters $ \theta $ of LA-GCN, we define the multi-task objective as:
\begin{equation}
	\label{formula:la-gcn loss}
	\mathop{\text{arg min}}_{\theta} \; ( L({f_\theta}^{pri}(x), \hat y)+ \lambda L({f_{\theta_n}}^{aux}(x), \hat y) ),
\end{equation}
where $ \lambda $ is the hyper-parameter used to adjust the $ L_{aug} $ weight. When testing, LA-GCN will drop the auxiliary branch to use the main branch for prediction.
 
The PC-AC module forces the features to be correctly classified is a relatively complex task, and it plays a role in regularization to some extent. During testing, the PC-AC module was removed without affecting the inference speed of the model. Meanwhile, T-C belongs to the fully connected graph, which contains relatively well-established inter-node relationships in each class. This property of T-C allows the PC-AC module to alleviate the over-smoothing problem \cite{ijcai_WangY0L21_MultiHop,icml_Xhonneux20_Continuous,aaai_Liu19_GeniePath,AAAI18_Li_DeeperInsights} common to GCN networks. Class-specific semantic information in T-C also improved the recognition rate of similar process actions, as shown in Fig. \ref{fig:pca}: ``reading'' and ``writing'' improved by 9.16\% and 8.46\%, respectively.

\section{Experiments}
\subsection{Datasets}
\textbf{NTU RGB+D.} The dataset \cite{Shahroudy_2016_NTURGBD} contains 56,880 skeleton action sequences and has 60 classes, which can be categorized into daily, healthy, and interactive behaviors. All action data were captured simultaneously by three Microsoft Kinect v2 cameras from different angles. Two evaluation schemes are presented in this paper:1) cross-subject (X-sub), where the training set is the data collected from 20 subjects and the data from the remaining 20 is the test set. 2) cross-view (X-view), using the data captured by the number 2 and 3 cameras as the training set and the data captured by the number 1 camera as the test set.

\noindent\textbf{NTU RGB+D 120.} The most commonly used is the incremental dataset NTU RGB+D 120 \cite{Liu_2020_NTURGBD120} of NTU RGB+D. It has 120 classes, 106 subjects, and 114,480 skeleton sequences. All of these actions were acquired by three cameras. Two benchmarks were introduced to evaluate model performance in NTU RGB+ D120: 1) crossover subjects (X-sub), which, as in NTU RGB+D, requires differentiation between two groups of subjects, and each group consists of 53 volunteers. 2) crossover setup (X-setup), in which data are acquired in different configurations. The training set is the even configuration data, and the test set is the odd configuration data.

\noindent\textbf{NW-UCLA.} The dataset \cite{Wang_2014_NWUCLA} contains ten basic human actions and 1494 video clips from 10 actors. All data were obtained from three Kinect cameras captured simultaneously. Following the evaluation protocol introduced in \cite{Wang_2014_NWUCLA}, the data from the first two cameras are used as the training set, and the data from the last camera as the test set.

\subsection{Training}
All experiments used the PyTorch deep learning framework \cite{nips_Paszke19_PyTorch} on 2$\times$ NVIDIA RTX 3090 GPUs. We used SGD with Nesterov momentum (0.9) as an optimizer with a weight decay of 0.0004. The entire training epoch was 110. The first five epochs were used with a warm-up strategy \cite{He_2016_warmup} to stabilize the training process. The primary learning rate is 0.1 and is reduced by 0.1 at 90 and 100 epochs. On NTU RGB+D and NTU RGB+D 120, the batch size of the experiment is 200, and the data are processed similarly \cite{Zhang_2020_SGN,Chen_2021_CTRGCN}. For NW-UCLA, the batch size was 64, and the same data preprocessing method as in \cite{Cheng_2020_ShiftGCN} was used. The above configuration was used for all subsequent experiments if not explicitly stated.
The source code of LA-GCN is publicly available on \url{https://github.com/damNull/LAGCN}.

\textbf{Overall Architecture.}
The encoder consists of nine basic blocks with 64-64-64-128-128-128-256-256-256 channels. The time dimension is reduced to half-time in blocks four and six. The skeleton sequence is first subjected to a feature transformation to obtain an embedding representation of the nodes, and the transformation uses a fully connected layer. This representation is then passed through the spatial and temporal modules to extract the sequence features.

\subsection{Results}
\begin{table}[!t]
	\caption{The comparison of Top-1 accuracy (\%) on the NTU RGB+D \cite{Shahroudy_2016_NTURGBD} benchmark.\label{table:ntu60}}
	\centering
	\begin{tabular}{c||c||c}
		\hline
		Method           & X-Sub & X-View \\
		\hline
		VA-LSTM \cite{pami_Zhang19_VALSTM} & 79.4 & 87.6    \\ 
		ST-GCN \cite{Yan_2018_STGCN}       & 81.5 & 88.3    \\
		AS-GCN \cite{Li_2019_ASGCN} 	   & 86.8 & 94.2    \\
		2s-AGCN \cite{Shi_2019_2sAGCN} 	   & 88.5 & 95.1    \\
		AGC-LSTM \cite{Si_2019_RNN}        & 89.2 & 95.0    \\
		Directed-GNN \cite{Shi_2019_4sDirectedGNN} & 89.9          & 96.1             \\
		ST-TR \cite{Plizzari_2021_Transformer}     & 90.3 		   & 96.3             \\
		Shift-GCN \cite{Cheng_2020_ShiftGCN} 	   & 90.7          & 96.5             \\
		DC-GCN+ADG \cite{Cheng_2020_DCGCN} 		   & 90.8          & 96.6             \\
		Dynamic-GCN \cite{Ye_2020_dynamicGCN} 	   & 91.5          & 96.0             \\
		MS-G3D \cite{Liu_2020_MSG3D}           	   & 91.5          & 96.2             \\
		DDGCN \cite{Korban_2020_DDGCN}             & 91.1          & 97.1             \\
		MST-GCN \cite{Chen_2021_MSTGCN}            & 91.5          & 96.6             \\
		EfficientGCN \cite{pami_SongZSW23_EfficientGCN} & 92.1 & 96.1  \\
		CTR-GCN \cite{Chen_2021_CTRGCN}  & 92.4 & 96.8  \\
		Info-GCN \cite{Chi_2022_InfoGCN} & 93.0 & 97.1  \\
		\hline
		LA-GCN (ours) & \bf 93.5 & \bf 97.2  \\
		\hline
	\end{tabular}
\end{table}
	
\begin{table}[!t]
	\caption{The comparison of Top-1 accuracy (\%) on the NTU RGB+D 120 \cite{Liu_2020_NTURGBD120} benchmark.\label{table:ntu120}}
	\centering
	\begin{tabular}{c||c||c}
		\hline
		Method            & X-Sub & X-Set \\
		\hline
		Part-Aware LSTM \cite{Shahroudy_2016_NTURGBD} & 26.3    & 25.5  \\
		ST-LSTM \cite{Liu_2016_ST_LSTM} 		& 55.7 & 57.9     \\
		RotClips+MTCNN \cite{Ke_2018_RotClip}   & 62.2 & 61.8     \\
		ST-GCN \cite{Yan_2018_STGCN}            & 70.7 & 73.2     \\
		2s-AGCN \cite{Shi_2019_2sAGCN} 			& 82.9 & 84.9     \\
		SGN \cite{Zhang_2020_SGN} 				& 82.9 & 84.9     \\
		ST-TR \cite{Plizzari_2021_Transformer}  & 85.1 & 87.1     \\
		Shift-GCN \cite{Cheng_2020_ShiftGCN} 	& 85.9 & 87.6     \\
		MS-G3D \cite{Liu_2020_MSG3D} 			& 86.9 & 88.4     \\
		Dynamic-GCN \cite{Ye_2020_dynamicGCN}   & 87.3 & 88.6     \\
		MST-GCN \cite{Chen_2021_MSTGCN} 		& 87.5 & 88.8     \\
		EfficientGCN \cite{pami_SongZSW23_EfficientGCN} & 88.7 & 88.9  \\
		CTR-GCN \cite{Chen_2021_CTRGCN}  & 88.9 & 90.6  \\
		Info-GCN \cite{Chi_2022_InfoGCN} & 89.4 & 90.7  \\
		\hline			
		LA-GCN (Joint)  	& 86.5  & 88.0     \\
		LA-GCN (Joint+Bone) & 89.7  & 90.9     \\
		LA-GCN (4 ensemble) & 89.9  & 91.3     \\
		LA-GCN (6 ensemble) & \bf 90.7  & \bf 91.8     \\
		\hline
	\end{tabular}
\end{table}

\begin{table}[!t]
	\caption{The comparison of Top-1 accuracy (\%) on the NW-UCLA \cite{Wang_2014_NWUCLA} benchmark.\label{table:ucla}}
	\centering
	\begin{tabular}{c||c}
		\hline
		Methods            & Top1 \\
		\hline
		Lie Group \cite{Veeriah_2015_LieGroup}        & 74.2 \\
		Actionlet ensemble \cite{Wang_2014_Actionlet} & 76.0 \\
		HBRNN-L \cite{Du_2015_Hierarchical}           & 78.5 \\
		Ensemble TS-LSTM  \cite{Lee_2017_TSLSTM}      & 89.2 \\
		\hline
		AGC-LSTM \cite{Si_2019_RNN} 			& 93.3 \\
		4s Shift-GCN \cite{Cheng_2020_ShiftGCN} & 94.6 \\
		DC-GCN+ADG \cite{Cheng_2020_DCGCN} 		& 95.3 \\
		CTR-GCN \cite{Chen_2021_CTRGCN} 		& 96.5 \\
		InfoGCN \cite{Chi_2022_InfoGCN} 		& 97.0 \\
		\hline
		LA-GCN (ours)   & \textbf{97.6} \\
		\hline
	\end{tabular}
\end{table}
Many state-of-the-art methods use a multi-stream fusion strategy. Specifically, the final results of the experiments fuse four modes (4s), namely, joint, bone, joint motion, and bone motion. They are the original joint coordinates, the vector obtained by differencing the coordinates of the joints with physical connections, the differentiation of the joint in the time dimension, and the differentiation of the bone in the time dimension, respectively. After training a model for each data stream, the softmax scores of each data stream are summed up as the final score during inference. For a fair comparison, we used the same setup as in \cite{Chi_2022_InfoGCN,Chen_2021_CTRGCN,Cheng_2020_ShiftGCN,Ye_2020_dynamicGCN}.

Our models are compared with state-of-the-art methods on the datasets NTU RGB+D \cite{Shahroudy_2016_NTURGBD}, NTU RGB+D 120 \cite{Liu_2020_NTURGBD120}, and NW-UCLA \cite{Wang_2014_NWUCLA}, respectively, and the experimental results are displayed in Table \ref{table:ntu60}, \ref{table:ntu120}, and \ref{table:ucla}. Our approach achieves state-of-the-art performance on three datasets under almost all evaluation benchmarks. We use the proposed new ``bone'' representation for 6s integration. The bones used for training 5s and 6s were selected from the proposed prompt 2 (p2) and prompt 5 (p5), respectively (see Appendix). Compared to the joint stream (86.5\%), our method improved by 3.2\%, 3.4\%, and 4.2\% for 2s, 4s, and 6s, respectively. On NTU-RGB+D 120, according to the same settings, our 4s fusion model is 0.6\% and 1.0\% higher than CTR-GCN 4s \cite{Chen_2021_CTRGCN} and 0.5\% higher than both InfoGCN 6s \cite{Chi_2022_InfoGCN}. Our 6s ensemble is 1.1\% and 1.3\% higher than InfoGCN 6s. Notably, our method is the first to effectively aid topological modeling using the knowledge of large-scale pre-trained models. 

\subsection{Ablation Study}
In this section, we analyze different configurations and ablation studies for each component of LA-GCN on the X-sub benchmark of the NTU RGB + D120 \cite{Liu_2020_NTURGBD120} dataset.

\begin{table}[!t]
	\caption{(i) The comparison of accuracy without and with PC-AC, where the $ \lambda $ of $ L_{aug} $ is 0.2. (ii) Contribution of different text prompts.\label{table:pcac}}
	\centering
	\begin{tabular}{c||c}
		\hline
		Methods  & Top1 \\
		\hline
		w/o $L_{aug}$ & 84.9\\
		$ L_{total} $ & $ {\bf 86.1}^{\uparrow 1.2} $\\
		\hline
		p1: [J] function in [C].                      & $ {85.6}^{\uparrow 0.7} $ \\
		p2: What happens to [J] when a person is [C]? & $ {85.8}^{\uparrow 0.9} $\\
		p3: What will [J] act like when [C]?          & $ {\bf 86.1}^{\uparrow 1.2} $\\
		p4: When [C][J] of human body.                & $ {85.5}^{\uparrow 0.6} $\\
		p5: When [C] what will [J] act like?          & $ {85.7}^{\uparrow 0.8} $\\
		p6: When a person is [C], [J] is in motion.   & $ {85.5}^{\uparrow 0.6} $\\
		\hline
	\end{tabular}
\end{table}

\begin{table}[!t]
	\caption{(i) The comparison of accuracy without and with PC-AC, where the $ \lambda $ of $ L_{aug} $ is 0.2. (ii) Contribution of different text prompts.\label{table:mhagc}}
	\centering
	\begin{tabular}{c||c}
		\hline
		Methods & Top1 \\
		\hline
		$\bar{A}$ & 86.1 \\
		$ \bar {\mathcal{A}} $ & \bf 86.5\\
		\hline
		$ 1\text{-}hop:\bar{A} $ & 86.1 \\
		$ 2\text{-}hop $ & 86.3 \\
		$ 3\text{-}hop $ & OOM \\
		\hline
		$ {2\text{-}hop}^{1st} $ & 86.1 \\
		$ {3\text{-}hop}^{1st} $ & 86.2 \\
		$ {4\text{-}hop}^{1st} $ & \bf 86.5 \\
		$ {5\text{-}hop}^{1st} $ & 85.9 \\
		\hline
	\end{tabular}
\end{table}

\subsubsection{Effectiveness of PC-AC} The validation of the PC-AC module in Sec. \ref{sec:pc-ac} includes the overall improvement due to $ L_{aug} $ and the performance of different generated prompts for the class exemplar topology T-C graphs. The model accuracy steadily improves by adding $ L_{aug} $ shown in Table \ref{table:pcac}.
There are differences between the T-C graphs obtained from different textual cues, and we design the prompt text p as a complete sentence containing action names [C] and joints [J]. For sample, p1 emphasizes ``nodes function,'' and p2 describes the change of nodes state, as demonstrated in Appendix Table \ref{tab:lambda}. The best result was obtained by p3: ``What will [J] act like when [C]'' with a 1.2\% improvement.

\begin{figure}[h]
	\centering
	\includegraphics[width=\linewidth]{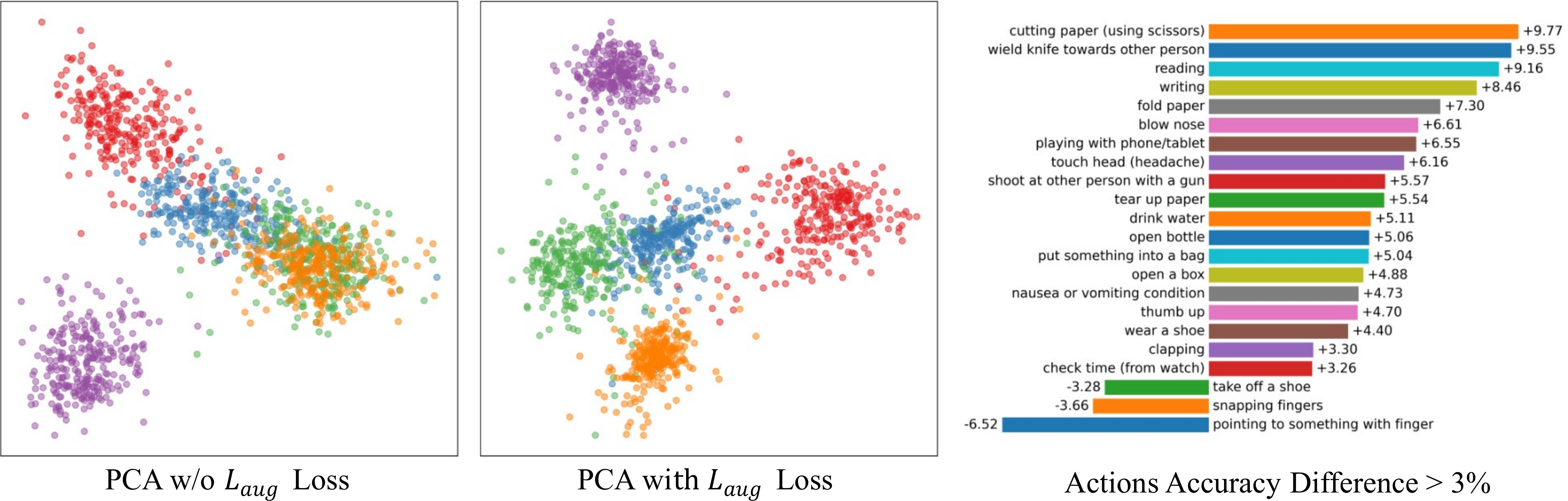}
	\caption{The left and middle subplots show the PCA projections of potential representations with and without the auxiliary loss $ L_{aug} $ of the PC-AC. The five action classes visualized were randomly selected from NTU RGB+D 120. The right subplot shows the visualization of actions with a greater than 3\% change in accuracy.}
	\label{fig:pca}
\end{figure}

We compare action representations based on trained models with and without PC-AC by principal component analysis (PCA \cite{Maaten2008VisualizingDU}), as shown in Fig. \ref{fig:pca}. Compared to the case without $ L_{aug} $ loss, the potential representations learned with the aid of the category prior present a class-conditional distribution with a more apparent separation and perform better in the subspace differentiation of process-similar actions. In addition, it makes the intra-class features more convergent.

\begin{figure*}[h]
	\centering
	\includegraphics[width=0.85\linewidth]{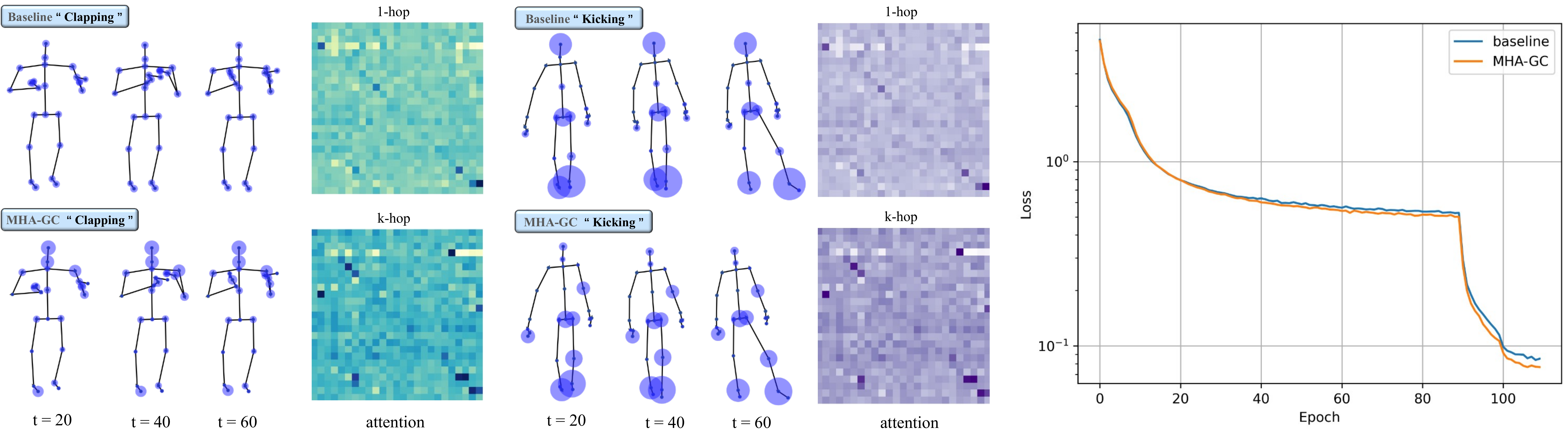}
	\caption{The left and middle subplots visualize the attention and its matrices of two actions, "clapping" and "kicking," with and without MHA-GC. The action selection focuses on the upper and lower body, respectively. Each action has a sampling interval of 20 frames. The larger the attention value, the larger the circle's radius centered on the joint. The right subplot shows the visualization of the loss function before and after adding MHA-GC.}
	\label{fig:vis-topo}
\end{figure*}

\subsubsection{Improved Interaction Between Neighbor Nodes.}
\label{sec:mha_abl}
Based on the above baseline, we design two schemes in Table \ref{table:mhagc} to replace the original $ 1\text{-}hop $ GC.
The first one uses MHA-GC for all blocks, and the other replaces only the $ 1\text{-}hop $ GC in the first layer. Then, we select $ k $ on top of that. We observe that the message passing performance of multi-hop GC for both schemes is strictly better than that of $ 1\text{-}hop $ GC. 
However, option I suffer from memory overflow at $ 3\text{-}hop $. Due to the relatively simple topology of the skeleton, complete replacement is somewhat oversupplied. The best performance in scheme II is 86.5\% at $ k = 4 $. We visualize the MHA in Fig. \ref{fig:vis-topo}. The multi-hop learnable topology greatly improves the description of the node relationships.
The loss curve on the right side shows that MHA-GC can accelerate the model convergence.

\textbf{Comparison of Complexity with Other Models}
As shown in the Fig. \ref{fig:la-gc}, $ \omega_0 = \beta $ and $ \bar{A}^0 = I $, $ \bar{\mathcal{A}}F^l $ is an approximation of below implementation when the condition  k-hop $ \to \infty $ is satisfied:
\begin{equation}
	\label{eq:dash_implement}
	F^{k+1} = (1 - \beta)\bar{A}F^{k} + \beta F^l,
\end{equation}
where $ 0 \leq k < K $. If $ k \to \infty $ , then $ lim_{k \to \infty} \sum^{k}_{i=0} \omega_i = 1 $ and $ \omega_i > 0 $, with the proposition $ lim_{K \to \infty}F^K = \bar{\mathcal{A}}F^l $ holds. The proof of this is given in the Appendix. This approximation indicates that the model complexity is in the same magnitude order as the previous attention-based SOTA approach \cite{Cheng_2020_DCGCN,Liu_2020_MSG3D,Chi_2022_InfoGCN}.
We compare the model and computational complexity with these methods on NTU-RGB+D 120 \cite{Liu_2020_NTURGBD120}. As shown in Table \ref{tab:complex}, our model balances complexity and final accuracy well. Our model is 1.4\% and 1.6\% more accurate than the previous SOTA InfoGCN and CTR-GCN while having the same level GFLOPs.
\begin{table}
	\caption{Comparison of computational and model complexity of the state-of-the-arts on NTU RGB+D 120 dataset.\label{tab:complex}}
	\centering
	\begin{tabular}{c||c||c||c}
		\toprule
		Methods & X-Sub & GFLOPs & Param (M)  \\
		\midrule
		DC-GCN \cite{Cheng_2020_DCGCN}  & 84.0 & 1.83 & 3.37 \\
		MS-G3D \cite{Liu_2020_MSG3D} & 84.9 & 5.22 & 3.22 \\
		CTR-GCN \cite{Chen_2021_CTRGCN} & 84.9 & 1.97 & 1.46 \\
		InfoGCN \cite{Chi_2022_InfoGCN} & 85.1 & 1.84 & 1.57 \\ 
		\midrule
		Ours    & \bf 86.5 & \bf 1.76 & \bf 1.46 \\
		\bottomrule
	\end{tabular}
\end{table}

\section{Limitations}
Despite the good results of the proposed LA-GCN in experiments, the attempts to acquire and apply LLM prior knowledge information are more oriented towards manual design. Further integration of language model assistance with self-supervised learning should be exciting. In addition, our skeleton modal representation is obtained based on the statistics of the existing dataset. It would also be interesting to see if new observations and conclusions can be made on more categories of datasets. Finally, expanding LA-GCN to in-the-wild settings is also worth thinking about.

\section{Conclusion}
We present LA-GCN, a representation learning framework aided by prior knowledge of language models. It builds on the cognitive neuroscience of human action recognition, i.e., the recognition process requires multiple brain regions to work together.  
LA-GCN accomplishes processes of encoding critical information and prediction based on the category a priori knowledge assistance. 
We introduce a new ``bone'' multimodal representation guided by global a priori information for model integration. Further, we propose a new multi-hop attention-based graph convolution module, MHA-GC, and demonstrate that it can effectively improve information interaction efficiency for modeling learnable topologies. The ``bone'' representation and MHA-GC are jointly used to ensure crucial information encoding. Ultimately, our model exhibits state-of-the-art performance on three popular datasets for skeleton-based action recognition.


{\appendices
\section{Structure Approximation Proposition}
As described in Sec. \ref{sec:mha_abl}, the structure in the dashed box in Fig. \ref{fig:la-gc} MHA-GC sub-block can be approximated by our k-hop structure under certain conditions. The equation of the dash box implementation is $ {F}^{k+1} = (1 - \beta)\bar{A}{F}^{k} + \beta F^l, $ where $ k $ is the iterations number of attention diffusion, $ F^k $ and $ F^{k+1} $ is the state of $ F^l $ at the k-th and k+1-th hop. 
\begin{proposition}
	\label{prop:struc_approx}
	$ lim_{K \to \infty}F^K = \bar{\mathcal{A}}F^l $
\end{proposition}
\begin{proof}
	Let $ E^0 = F^l $, then Eq. \ref{eq:dash_implement} becomes:
	\begin{equation}
		E^{k+1} = (1-\beta)\bar{A}E^k + \beta E^0.
	\end{equation}
	
	Let $ K > 0 $ be the hop number, we approximate $ \hat{F}^l $ by $ E^K $. By recursion, we can obtain
	\begin{equation}
		\begin{split}
			E^K & = (1-\beta)\bar{A}E^{K-1} + \beta E^0 \\
			& = (1-\beta)\bar{A}((1-\beta)\bar{A}E^{K-2} + \beta E^0) + \beta E^0 \\
			& = (1-\beta)^2 \bar{A}^2 E^{K-2} + \beta(1-\beta)\bar{A}E^0 + \beta E^0 \\
			& = (1-\beta)^3 \bar{A}^3 E^{K-3} + \beta(1-\beta)^2 \bar{A}^2 E^0 \\
			& + \beta(1-\beta)\bar{A}E^0 + \beta E^0 \\
			& \cdots \\
			& = ((1-\beta)^K \bar{A}^K + \beta\sum_{i=0}^{K-1} (1-\beta)^i \bar{A}^{i})E^{0}. \\	    
		\end{split}
	\end{equation}
	
	That is $ E^K = ((1-\beta)^K \bar{A}^K + \beta\sum_{i=0}^{K-1} (1-\beta)^i \bar{A}^{i})F^l $. As $ \beta \in (0,1] $ and $ \bar{A}_{i,j}^{K} \in (0,1] $, when $ K \to \infty $, the term $ (1-\beta)^K \bar{A}^K \to 0 $. 
	
	Thus, $ lim_{K \to \infty}E^K = (\sum_{i=0}^{K-1} \beta(1-\beta)^i \bar{A}^{i})F^l = \bar{\mathcal{A}}F^l. $ 
\end{proof} 

The skeleton sequence contains a fixed number of edges $ \varepsilon $, and by the above approximation, we conclude that the complexity of multi-hop attention does not differ much from one-hop attention. The complexity can all be expressed as $  O(|\varepsilon|) $, where the number of hops $ K $ is the factor that determines the complexity of multiple hops. A better representation can be obtained for skeletal data $ 1 \le K \le 4 $ in Sec. \ref{sec:mha_abl}. Also, if the structural complexity of the graph increases, then $ K $ needs to increase as well.

\section{Mathematical Explanation of Effectiveness for MHA}
In addition to the intuitive conclusion from Table \ref{table:mhagc} that the expressiveness of multiple hops is better than that of one hop, we can also perform a spectral analysis \cite{Sandryhaila_2013_GraphFourier} of $ \bar{A} $ and $ \bar{\mathcal{A}} $ to obtain relevant proof.

The relationship between our multi-hop and one-hop attention is given in Eq. \ref{eq:relation_multi_one}: $ \bar{\mathcal{A}} = \sum_{i=0}^{k} \omega_i\bar{A}^i $, $ \omega_i = \beta(1-\beta)^i $, and $ \bar{A}^i $ is the power matrix. In spectral analysis, the Jordan decomposition of graph attention $ \bar{A} $ is $ \bar{A} = U \bar{\Lambda} U^{-1}  $, where the $ i $-th column of $ U $ is the eigenvector $ u_i $ of $ \bar{A} $, $ \bar{\Lambda} $ is the diagonal matrix, and each element in $ \bar{\Lambda} $ is the eigenvalue corresponding to $ u_i $, $ \bar{\Lambda}_{ii} = \lambda_i $. Then we have:
\begin{equation}
	\bar{\mathcal{A}} = \sum_{l=0}^{K} \omega_l\bar{A}^l = \sum_{l=0}^{K} \omega_l(U \bar{\Lambda} U^{-1})^l
\end{equation}
The eigenvectors of the power matrix $ \bar{A}^n $ are same as $ \bar{A} $, we can get $ \bar{A}^n = (U \bar{\Lambda} U^{-1})(U \bar{\Lambda} U^{-1})\cdots(U \bar{\Lambda} U^{-1}) = U \bar{\Lambda}^n U^{-1} $. In addition, the eigenvectors of $ \bar{A} + \bar{A}^2 $ are same as $ \bar{A} $. By recursion, the summation of the $ \bar{A}^l, l \in [0, K] $ has the same eigenvectors as $ \bar{A} $, we have:
\begin{equation}
	\bar{\mathcal{A}} = U(\sum_{l=0}^{K} \omega_l\bar{\Lambda}^l)U^{-1}.
\end{equation}
Therefore, we have an analogy that the eigenvectors of $ \bar{\mathcal{A}} $ and $ \bar{A} $ are also the same.
\begin{lemma}
	Let $ \bar{\lambda}_i $ and $ \lambda_i $ be the $ i $-th eigenvalues of $ \bar{\mathcal{A}} $ and $ \bar{A} $, respectively. Then, the eigenvalue relationship function is
	\begin{equation}
		\label{eq:eigenvalue_relation}
		\bar{\lambda}_i = \sum_{l=0}^{\infty} \omega_l{\lambda_i}^l = \sum_{l=0}^{\infty} \beta(1-\beta)^l{\lambda_i}^l = \frac{\beta}{1-(1-\beta)\lambda_i}
	\end{equation} 
\end{lemma}
\begin{proof}
	$ \bar{L} = I - Q^{-\frac{1}{2}}\bar{A}Q^{-\frac{1}{2}} $ is the symmetric normalized Laplacian of skeleton graph $ G $, where $ Q = diag([q_1,q_2,\cdots,q_N]), q_i =  \sum_{j=0}^{\bar{A}_{ij}} $ is the degree matrix. Since $ \bar{A} $ is the one-hop attention matrix of $ G $, if $ \bar{A} $ is generated by softmax, then $ q_i = 1 $ and thus $ Q = I $. Therefore, $ \bar{L} = I - \bar{A} $, and the eigenvalue of $ \bar{L} $ is $ \tilde{\lambda}_i = 1 - \lambda_i $. Also, the work \cite{Mohar_1991_Laplacian} proves that the value domain of eigenvalues $ \tilde{\lambda}_i $ of symmetric normalized Laplace matrix  is [0, 2]. Thus, $ -1 \le \lambda_i \le 1 $ and $ \beta \in (0, 1) $, we have $ |(1-\beta)\lambda_i| \le (1-\beta) < 1 $. When $ K \to \infty $, $ ((1-\beta)\lambda_i)^K \to 0 $ and $ \bar{\lambda}_i = lim_{K \to \infty} \sum_{l=0}^{K}\beta(1-\beta)^l{\lambda_i}^l $. Let $ R $ be $ \sum_{l=0}^{K}\beta(1-\beta)^l{\lambda_i}^l $ is calculated as follow
	\begin{equation}
		\begin{split}
			R & = \beta + \beta(1-\beta)\lambda_i + \beta(1-\beta)^2{\lambda_i}^2 + \cdots + \beta(1-\beta)^K{\lambda_i}^K \\
			& = \beta(1 + (1-\beta)\lambda_i + ((1-\beta)\lambda_i)^2 + \cdots + ((1-\beta)\lambda_i)^K) \\
			& = \beta \frac{1 - ((1-\beta)\lambda_i)^K}{1 - (1-\beta)\lambda_i}.
		\end{split}
	\end{equation}
	We get $ \bar{\lambda}_i = lim_{K \to \infty} R = lim_{K \to \infty} \frac{\beta(1 - ((1-\beta)\lambda_i)^K)}{1 - (1-\beta)\lambda_i} = \frac{\beta}{1-(1-\beta)\lambda_i} $. 
\end{proof}

The above proof is under the condition that $ \bar{A} $ is obtained by softmax, but here, we use the feature distance to compute $ \bar{A} $ to maintain semantic consistency with the GPR Graph (Sec. \ref{sec:encoding}). For this purpose, we visualize the degree matrix $ Q $ with value domain of eigenvalue less than 0.3. Thus, the eigenvalue $ \lambda_{Qi} $ of $ Q^{-1/2} $ has a value domain greater than 1, and $ -1 \le \lambda_{Qi}\lambda_{i}\lambda_{Qi} \le 1 $. Since $ \tilde{\lambda}_i \in [0, 2] $, $ | \lambda_{Qi}\lambda_{i}\lambda_{Qi}| \le 1 $. And $ \lambda_{Qi} \ge 1 $, we have $ |\lambda_{i}| \le 1 $. The condition $ |(1-\beta)\lambda_i| \le (1-\beta) < 1 $ still holds.

We further obtain the eigenvalue relations for the normalized laplacian graph of $ \bar{\mathcal{A}} $ and $ \bar{A} $ according to Eq. \ref{eq:eigenvalue_relation}:
\begin{equation}
	\label{eq:laplacian_eigenvalue_relation}
	\begin{split}
		\frac{\bar{\lambda}_{i}^{G}}{\lambda_{i}^{G}} & = \frac{1 - \bar{\lambda}_{i}}{\lambda_{i}^{G}} = \frac{1 - \frac{\beta}{1-(1-\beta)\lambda_i}}{\lambda_{i}^{G}} = \frac{1 - \frac{\beta}{1-(1-\beta)(1 - \lambda_{i}^{G})}}{\lambda_{i}^{G}} \\
		& = \frac{1}{\frac{\beta}{1 - \beta} + \lambda_{i}^{G}}.
	\end{split}
\end{equation}
If $ \lambda_{i}^{G} $ is large, the multi-hop Laplacian eigenvalues $ \bar{\lambda}_{i}^{G} $ less than one-hop $ \lambda_{i}^{G} $. Eq. \ref{eq:laplacian_eigenvalue_relation} shows that multi-hop can add smaller eigenvalues. The smaller the beta, the more pronounced the effect. It has been shown in the spectral analysis that low-frequency signals correspond to topological information of the graph, while higher eigenvalues correspond more to noise \cite{Andrew_2001_SpectralClustering,Gasteiger_2019_Diffusion}. Therefore, multi-hop learned representations can better describe the inter-node relationships.
\begin{figure}[h]
	\centering
	\includegraphics[width=0.85\linewidth]{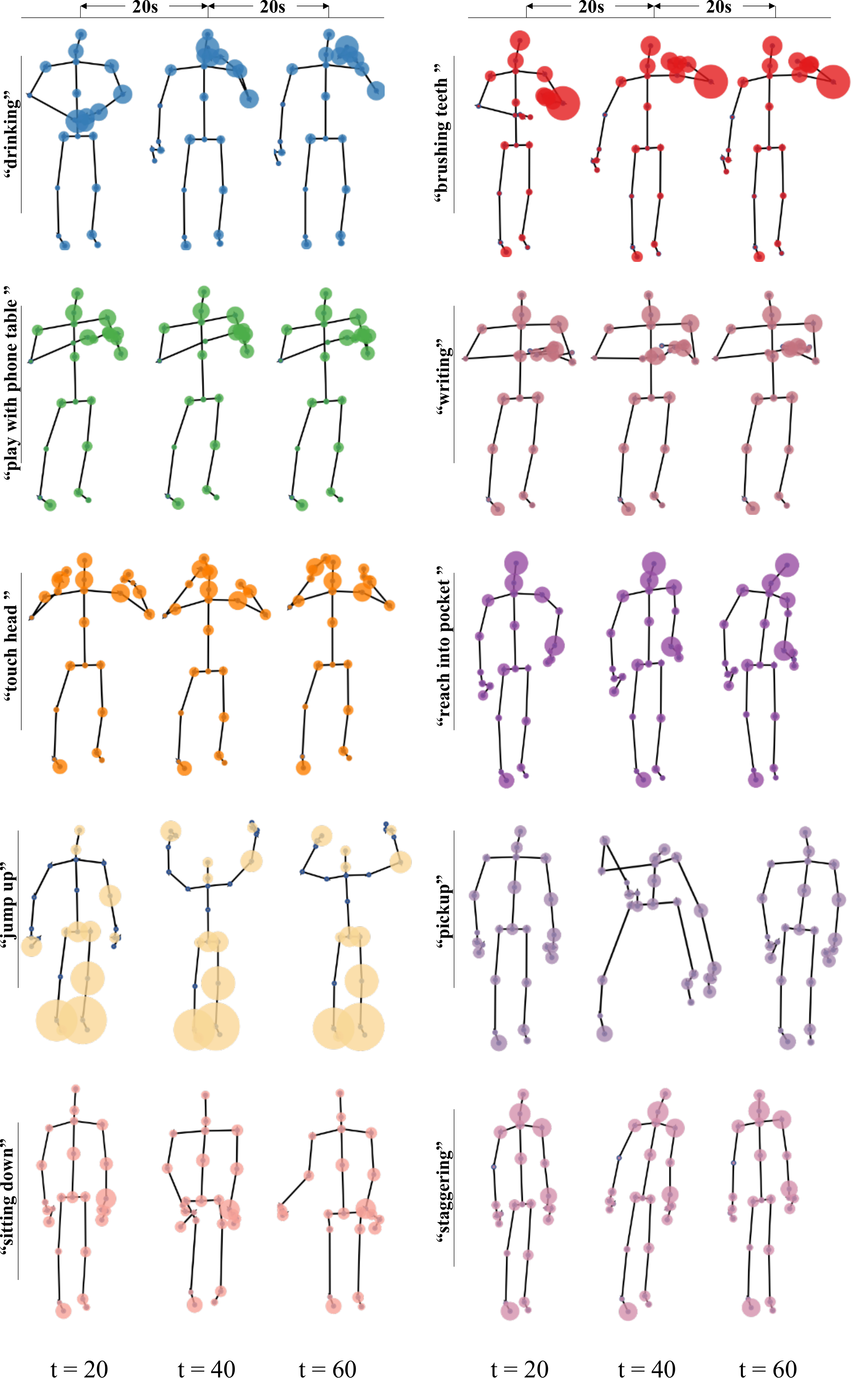}
	\caption{More action visualization for MHA, using different colors for different actions. Each action has a sampling interval of 20 frames. The larger the attention value, the larger the circle's radius centered on the joint.}
	\label{fig:more_mha_vis}
\end{figure}

In addition, we visualize the joint attention of MHA on more actions in Fig. \ref{fig:more_mha_vis}, including ``drinking,'' ``playing with phone table,'' ``touching the head,'' ``writing,'' ``brushing teeth,'' ``reach into pocket,'' ``jump up,'' ``pick up,'' ``sitting down,'' and ``staggering.'' Some actions focus on the upper body, such as the head and hands. Some actions focus on the lower body, such as the hip, knee, and foot. Our MHA can correctly show the strength of the relationship between these critical parts in different actions. The relationship between the left hand and the head was stronger for ``brushing'' than for ``drinking.'' In the two actions of ``writing'' and ``playing with phone tablet,'' MHA paid more vital attention to the "spine" during writing. In addition, MHA can effectively describe the remote interaction between hands and legs in "jump up" and between neck, shoulders, and feet in "staggering."

\section{Supplemental Experimental Results}
\subsection{Ablations on NTU RGB+D 120}
\begin{table}
	\caption{$\lambda$ selection of PC-AC with different text prompts T-C.\label{tab:lambda}}
	\centering
	\scalebox{0.9}{
		\begin{tabular}{c||c||c||c||c}
			\hline
			$\lambda$ &  0.1  & 0.2  & 0.3  & 0.5  \\
			\hline
			p1: [J] function in [C]. &  85.0 & 85.6 & 85.3 & 84.8 \\
			p2: What happens to [J] when a person is [C]? &  85.1 & 85.8 & 85.5 & 85.1 \\
			p3: What will [J] act like when [C]? &  85.3 & \bf 86.1 & 85.6 & 85.3 \\
			p4: When [C][J] of human body. &  85.0 & 85.5 & 85.0 & 84.8 \\
			p5: When [C] what will [J] act like? &  85.2 & 85.7 & 85.2 & 84.9 \\
			p6: When a person is [C], [J] is in motion. &  85.1 & 85.5 & 85.2 & 85.0 \\
			\hline
	\end{tabular}}
\end{table}

\textbf{Weighting Factors for $ L_{aug} $ in Total Loss.} We perform a grid search for the coefficients $ \lambda $ of $ L_{aug} $ and the prompt that generates T-C in PC-AC. As shown in Table \ref{tab:lambda}, $ \lambda=0.2 $ performs best for the class exemplars T-C generated using p3. Prompt p3 searches for possible state features of each node as the action is performed and calculates the correlation of features between nodes in T-C. The design has the same goal as the skeleton-based action recognition task and validates the importance of node correlation modeling for correct classification.

\begin{table}[!t]
	\caption{The PC-AC input and output configuration by Top1 Acc (\%). On the left is the result of selecting the input features $theta_n$. On the right is the pooling method's output logit $  \mathbb{R}^{N \times 1 \times cls \times V} \to \mathbb{R}^{N \times cls} $. \label{tab:input_output_PCAC}}
	\centering
	\begin{tabular}{cc||cc}
		\hline
		$\theta_n$ &  Acc  & Pool & Acc\\
		\hline
		$\theta_1$    &  84.7 & Avg    &  \bf 86.1 \\
		$\theta_2$    &  85.3 & Max    &  85.9 \\
		$\theta_3$    &  \bf 86.1 & Avg $\oplus$ Max &  85.4 \\
		$\theta_2 + \theta_3$   &  84.2 & Weighted sum   &  85.7 \\
		\hline
	\end{tabular}	
\end{table}

\noindent\textbf{Input and Output Configuration of The PC-AC Module.} PC-AC has the best results when the feature $ \theta_3 $ of the last stage of the main branch is selected as an input, and the experiments are shown in Table \ref{tab:input_output_PCAC} left. Our model consists of nine base blocks. These nine basis blocks are divided into three stages based on the feature dimensional descent process. $ \theta_3 $, $ \theta_2 $, and $ \theta_1 $ corresponds to the output features of the ninth, sixth, and third blocks. Except for the input feature selection, we further compared the different pooling methods used to reduce the feature dimensionality and get classification logits. Table \ref{tab:input_output_PCAC} right shows that the average pooling performs the best. The average relationship between joints is more meaningful as the category-by-category representation.

\begin{table}[!t]
	\caption{The accuracy (\%) of the five-stream (5s) and six-stream (6s) ensemble on the NTU RGB+D 120 X-Sub split between the traditional four-stream (4s) and the new ``bone'' representatives. The basis is the joint acc: 86.5\%. The model accuracy (\%) of the new ``bone'' representations \{$ B_{p1}, B_{p2}, B_{p3}, B_{p4}, B_{p5}, B_{p6} $\} on the NTU RGB+D 120 X-Sub split are 85.9, 86.0, 84.9, 85.9, 85.8, and 85.8, respectively.\label{tab:bone_ensemble_120_xsub}}
	\centering	
	\begin{tabular}{c||c||c}
		\toprule
		Modes & Components & Acc \\
		\midrule
		2s & Joint + Bone & $ {89.7}^{\uparrow 3.2} $ \\
		4s & J + B + JM + BM & $ {89.9}^{\uparrow 3.4} $ \\
		\midrule
		\multirow{6}{*}{5s} & 4s + $ B_{p1} $ & $ {90.1}^{\uparrow 3.6} $ \\
		& 4s + $ B_{p2} $ & $ {\bf 90.4}^{\uparrow 3.9} $ \\
		& 4s + $ B_{p3} $ & $ {90.3}^{\uparrow 3.8} $ \\
		& 4s + $ B_{p4} $ & $ {90.1}^{\uparrow 3.6} $ \\
		& 4s + $ B_{p5} $ & $ {\bf 90.4}^{\uparrow 3.9} $ \\
		& 4s + $ B_{p6} $ & $ {90.2}^{\uparrow 3.7} $ \\
		\midrule
		6s					& 4s + $ B_{p2} + B_{p5} $ & $ {\bf 90.7}^{\uparrow 4.2} $ \\
		\bottomrule
	\end{tabular}		
\end{table}

\noindent\textbf{Multi-stream Ensemble.} We integrated the skeletal modal representations generated by different prompts to observe the final performance of the multi-stream integration. We give the classification accuracy of the skeleton representation $ B_{p*} $ generated by each prompt and the ensemble results for 2s, 4s, 5s, and 6s in Table \ref{tab:bone_ensemble_120_xsub}. The 2s in Table \ref{tab:bone_ensemble_120_xsub} represent the integrated joint and bone models, the 4s represent the integrated 2s and their motion modalities, the 5s represent the integrated 4s and arbitrary $ B_{p*} $ modalities, and the 6s represent the integrated 4s and the two best-performing prompts $ B_{p2} $ and $ B_{p5} $ in the 5s. In Table \ref{tab:bone_ensemble_120_xsub}, we can see that the models trained with $ B_{p*} $ have a corresponding improvement in accuracy. Choosing the two best-performing $ B_{p*} $ for integration can lead to higher performance.

\subsection{Multi-stream Ensemble on All Datasets}
\begin{table*}
	\caption{Accuracies (\%) of the six streams ensemble on NTU RGB+D 120 X-Set split, NTU RGB+D 60 X-Sub and X-view split, and NW-UCLA \label{tab:120_xset_60_xsub_xview_nw}}
	\centering
	\begin{tabular}{c||c||c||c||c||c}
		\hline
		\multirow{2}{*}{Modes} & \multirow{2}{*}{Components} & NTU RGB+D 120 & \multicolumn{2}{c||}{NTU RGB+D 60} & NW-UCLA \\
		\cline{3-6}
		& & X-Set & X-Sub & X-view & Top-1 \\
		\hline
		\multirow{4}{*}{1s} & Joint & 88.0 & 90.5 & 95.5 & 95.7 \\
		& Bone  & 88.6 & 90.9 & 95.2 & 93.1 \\
		& $ B_{p2} $ & 87.1 & 89.6 & 93.8 & 90.7 \\
		& $ B_{p5} $ & 87.0 & 89.4 & 93.9 & 92.5 \\
		\hline
		2s & Joint + Bone & 91.0 & 92.3 & 96.6 & 96.3 \\
		4s & J + B + JM + BM & 91.3 & 93.0 & 97.1 & 96.8 \\   
		6s & 4s + $ B_{p2} + B_{p5} $ & \bf 91.8 & \bf 93.5 & \bf 97.2 & \bf 97.6\\
		\hline
	\end{tabular}
\end{table*}

The experimental results of the multi-stream integration of our method on different validation classifications of the three benchmark datasets are shown in Table \ref{tab:120_xset_60_xsub_xview_nw}. All 6s ensembles were done by the best-performing skeleton modalities  $ B_{p2} $ and $ B_{p5} $ on NTU RGB+D 120 X-Sub. It can be observed that the integration performance of the skeleton modalities generated by these two prompts is robust on all datasets.
Our method uses the sum of standard deviations as the criterion for selecting a new skeleton representation and requires that the newly selected ``bone'' be similar to the original skeleton representation. Moreover, the preserved bones are guided by a priori information. Although the model learned from the new skeleton representation does not perform as well as the original bone model, its good performance in ensemble experiments demonstrates that the variability of the new ``bone'' can effectively complement feature learning.

We give all the symbols in Sec. \ref{Sec:method} in Table \ref{tab:notations} to make it more accessible. 

\begin{table*}
	\caption{Summary of symbols in the methods section. \label{tab:notations}}
	\centering
	\scalebox{0.92}{
		\begin{tabular}{c||c||c||c}
			\toprule
			Location & Symbol & Notation Type & Descriptions  \\
			\midrule
			\multirow{6}{*}{Sec. \ref{sec:GPR} GPR Graph} & $ M $ & Constant & The number of action classes \\
			& $ N $ & Constant & The number of human joint nodes \\
			& C $ \in \mathbb{R}^{N \times C} $ & Variable & The LLM output text features for all joint tokens of each action category \\
			& J $ \in \mathbb{R}^{1 \times C} $ & Variable & Node features \\
			& $ C $ & Constant & Feature dimension of J \\
			& $ \text{J}^{CoCLS} \in \mathbb{R}^{N \times C} $ & Variable & Category centroid vector \\
			\midrule
			\multirow{3}{*}{Sec. \ref{sec:multimodalrp} New ``bone'' data} & $ B \in \mathbb{R}^{N \times N} $ & Variable & The bone matrix containing the position relationship between the source and target joints\\
			& $ \widetilde{B} $ & Variable & New ``bone'' matrix \\
			& $ g(\cdot) $      & Function & The selection function of links in $ \widetilde{B} $ \\
			\midrule
			\multirow{12}{*}{Sec. \ref{sec:encoding} MHA-GC} & X & Random Variable & The input skeleton sequence representation \\
			& $ N $ & Constant & Batch size \\
			& $ C $ & Constant & The number of input/output channels \\
			& $ T $ & Constant & Frame length of the sequence X \\
			& $ V $ & Constant & The number of joints \\
			& $ F^l \in \mathbb{R}^{N \times C \times T \times V}$ & Learnable Parameter & Hidden layer representation of GCN \\
			& $ \bar{A}^l $ & Learnable Parameter & First-order attention on all nodes \\
			& $ \tilde{A}^l $ & Learnable Parameter & First-order neighborhood information \\
			& $ \dot{A}^l $ & Learnable Parameter & The shared skeleton topology \\
			& $ \bar{\mathcal{A}}^l $ & Learnable Parameter & Multi-hop attention on all nodes \\
			& $ \gamma $ & Learnable Parameter & The learnable refinement weight of $ \tilde{A}^l $ \\
			& $ \omega $ & Control Parameter & The decay factor of $ \bar{A} $ \\
			\midrule
			\multirow{7}{*}{Sec. \ref{sec:pc-ac} PC-AC} & $ \theta $ & Learnable Parameter & Main branch feature \\
			& $ \theta_n $ & Learnable Parameter & Features shared by the main branch to the auxiliary branch  \\
			& $ f_{\theta}^{pri}(x) $ & Random variable & The main branch prediction \\
			& $ f_{\theta}^{aux}(x) $ & Random variable & The auxiliary branch prediction \\
			& T-C & Graph & Class topology exemplar \\
			& $ L_{aug} $ & Loss & Auxiliary Classification loss \\
			& $ \lambda $ & Control Parameter & Weight parameter of $ L_{aug} $ \\
			\bottomrule
	\end{tabular}}
\end{table*}

}


 
%

%
\bibliographystyle{IEEEtran}
\bibliography{sample-base}

\newpage

%
%
%
%

\vfill

\end{document}